\documentclass[sigconf,nonacm]{acmart} 
\AtBeginDocument{%
  \providecommand\BibTeX{{%
    \normalfont B\kern-0.5em{\scshape i\kern-0.25em b}\kern-0.8em\TeX}}}

\setcopyright{none}

\usepackage[T1]{fontenc}
\usepackage{tikz}
\usetikzlibrary{positioning, arrows.meta, calc, fit}
\usepackage{pgfplots}
\pgfplotsset{compat=1.18}
\usepackage{graphicx}
\usepackage[table]{xcolor}
\usepackage{adjustbox}
\usepackage{booktabs}
\usepackage{makecell}
\usepackage{graphicx}
\usepackage{booktabs}
\usepackage{makecell}
\usepackage{graphicx}
\usepackage{multirow}
\usepackage{pifont}

\makeatletter
\providecommand{\@copyrightspace}{}
\renewcommand{\@copyrightspace}{}
\makeatother

\usepackage{booktabs}   
\usepackage{tabularx}
\usepackage{array}
\usepackage{multirow}   
\usepackage{makecell}
\usepackage{graphicx}

\usepackage{amsmath}    
\usepackage{amsfonts}   
\usepackage{amsthm}     
\usepackage{nccmath}    
\usepackage{bm}         
\usepackage{bbm}        

\usepackage[ruled, vlined, linesnumbered]{algorithm2e}
\usepackage{algorithmic}

\usepackage{caption}
\usepackage{subcaption} 
\usepackage{tikz}
\usepackage{pgfplots}
\pgfplotsset{compat=1.18}

\definecolor{cBar}{HTML}{B8D8A8}
\definecolor{cOurs}{HTML}{1F4E79}
\definecolor{cRAG}{RGB}{88,102,230}
\definecolor{cAMEM}{RGB}{60,160,90}
\definecolor{cSteel}{RGB}{220,90,90}
\definecolor{cGrid}{RGB}{180,180,180}
\definecolor{cNaive}{HTML}{4C78A8}
\definecolor{cRandom}{HTML}{F2A65A}
\definecolor{cRule}{HTML}{59A14F}
\definecolor{cAMEM}{HTML}{9C6ADE}
\definecolor{cMemoRAG}{HTML}{8C8C8C}
\definecolor{cConMem}{HTML}{D62728}

\usepackage{tikz}
\usetikzlibrary{calc}

\usepackage{enumerate}
\usepackage{enumitem}   
\usepackage{balance}    
\usepackage{afterpage}
\usepackage[normalem]{ulem}
\useunder{\uline}{\ul}{} 
\newcolumntype{Y}[1]{%
  >{\hsize=#1\hsize\linewidth=\hsize
    \raggedright\arraybackslash}X%
}

\usepackage{hyperref}

\AtBeginDocument{%
  \providecommand\BibTeX{{%
    Bib\TeX}}}

\newtheorem{assumption}{Assumption}
\newtheorem{corollary}{Corollary}
\newtheorem{Defination}{Definition}
\newtheorem{example}{Example}

\newcommand\figref[1]{Fig.~\ref{#1}}

\newcommand\secref[1]{Sec.~\ref{#1}}

\newcommand{\fakeparagraph}[1]{\vspace{1mm}\noindent\textbf{#1.}}

\ifodd 1
\newcommand{\TODO}[1]{\textbf{\color{red}{TODO: #1} }}

\else
\newcommand{\TODO}[1]{}

\fi

\newcommand{\sysname}{\textsf{ConMem}\xspace}

\definecolor{fullrag}{HTML}{7E57C2}
\definecolor{randomc}{HTML}{E45756}
\definecolor{rulec}{HTML}{72B7B2}
\definecolor{steelmem}{HTML}{54A24B}

\makeatletter
\renewcommand{\@authorfont}{\Large}
\renewcommand{\@affiliationfont}{\footnotesize}
\makeatother

\begin{document}


\title{
    ConMem: Contribution-Aware Memory for Long-Horizon Manufacturing Inspection Logs
}

\author{Bingchen Liu}
\email{lbcraf2018@126.com}
\orcid{0000-0001-5041-8593}
\affiliation{
  \institution{Shandong University, School of Software}
  \city{Jinan}
  \state{Shandong}
  \country{China}
}

\author{Yuanyuan Fang}
\email{fyy@bu.edu}
\affiliation{
  \institution{Boston University, Metropolitan College}
  \city{Boston}
  \state{MA}
  \country{USA}
}

\author{Lei Liu}
\email{120242227112@ncepu.edu.cn}
\affiliation{
  \institution{North China Electric Power University, School of Control and Computer Engineering}
  \city{Beijing}
  \state{Beijing}
  \country{China}
}

\author{Guangyuan Dong}
\email{guangyuan@u.nus.edu}
\affiliation{
  \institution{National University of Singapore, Department of Statistics and Data Science}
  \city{Singapore}
  \state{Singapore}
  \country{Singapore}
}

\author{Xing Fu}
\email{fuxing@bu.edu}
\affiliation{
  \institution{Boston University, Metropolitan College}
  \city{Boston}
  \state{MA}
  \country{USA}
}

\author{Yuanyuan Gao}
\email{202635611@mail.sdu.edu.cn}
\affiliation{
  \institution{Shandong University, School of Software}
  \city{Jinan}
  \state{Shandong}
  \country{China}
}

\author{Shuyue Wei}
\email{weishuyue@sdu.edu.cn}
\affiliation{
  \institution{Joint SDU-NTU Centre for Artificial Intelligence Research (C-FAIR), Shandong University}
  \city{Jinan}
  \state{Shandong}
  \country{China}}

\author{Xin Li}
\email{lx@sdu.edu.cn}
\affiliation{%
  \institution{Shandong University, School of Software}
  \city{Jinan}
  \state{Shandong}
  \country{China}}

\author{Xiangtian Meng}
\email{mxt99@163.com}
\affiliation{%
  \institution{Rizhao Steel Holding Group Co., Ltd.}
  \city{Rizhao}
  \state{Shandong}
  \country{China}
}

\renewcommand{\shortauthors}{Liu et al.}

\makeatletter
\renewcommand\footnotetextcopyrightpermission[1]{}
\def\ACM@checkcopyright{}
\makeatother

\begin{abstract}
    Long-horizon steel-equipment inspection requires reasoning over heterogeneous records accumulated across repeated inspection cycles.
    Existing retrieval-augmented generation systems treat historical logs as a static corpus and retain records without estimating their diagnostic value, failing to report early risk. 
    To this end, we propose ConMem, a contribution-aware memory framework for LLM-assisted equipment inspection, supporting a human-in-the-loop early-risk screening system.
    Specifically, our ConMem first segments inspection logs into functional evidence units, then estimates each memory unit’s contribution to downstream diagnosis through a Shapley-style estimation, and finally retains high-value evidence under a constrained memory budget.
    In experiments, we evaluate ConMem on real-world dataset and ConMem achieves 76.0\% QA accuracy, exceeding the strongest directly comparable baseline.
    Relative to the naive 8K-context LLM baselines, it reduces the average number of input tokens by 88.2\% and response time by 86.6\%.
    Ablation studies also show that the functional-role-aware segmentation and contribution-based valuation are helping prioritize weak degradation signals for targeted field inspection.
    Practical deployments further confirm that ConMem retains the weak early signal across three inspection cycles, providing an early-stage seal-wear alert targeted for on-site inspectors.
    Source codes are available at \url{https://anonymous.4open.science/r/ConMem-F1B0/}.

\end{abstract}





\maketitle

\section{INTRODUCTION}
    Modern manufacturing facilities rely on periodic routine inspections of core production assets, such as motors, pumps, bearings, and conveyors, to detect operational risks early and prevent economically expensive unplanned downtime~\cite{bib1005}.
    Thus, equipment inspection plays a critical role in modern manufacturing, such as in the steel industry.
    However, traditional manual inspection methods are usually labor-intensive, subjective, and often fail to capture early signs of equipment degradation.
    With the rapid advancement of large language models (LLMs), there is growing interest in performing equipment inspection with the assistance of LLM-based agents, which enable retrieval and reasoning over historical inspection records \cite{bib1001, bib1002}.
    We show a typical example of LLM-assisted equipment inspection for human-in-the-loop early-risk screening system in the steel manufacturing industry in \figref{fig:workflow}.
\begin{figure}[!htb]
\vspace{-1em}
\centering 
\includegraphics[width=1\linewidth]{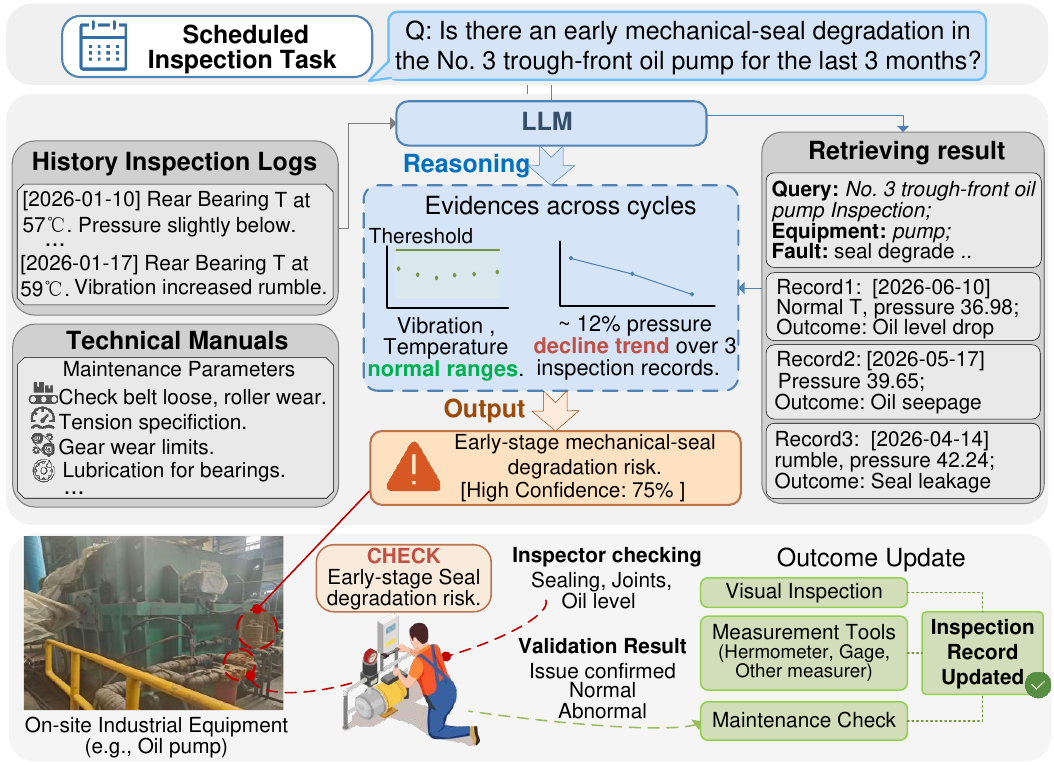}
\vspace{-1em}  
\caption{A typical equipment inspection workflow.}
\label{fig:workflow}
\vspace{-1.5em}
\end{figure}

\begin{example}[Real-world Equipment Inspection]
        A typical industrial inspection workflow consists of three stages. 
    \ding{172} The system receives a query of a scheduled inspection task, from which the LLM identifies the target equipment and inspection objectives, and then retrieves the corresponding technical manuals together with historical inspection records. 
    \ding{173}The system then analyzes heterogeneous inspection evidence, including equipment status, safety anomalies, maintenance action, and technical manuals, to evaluate equipment health. 
    For instance, the LLM retrieves historical inspection records from memory and identifies a pressure drop over the previous three inspection cycles. 
    By combining this degradation trend with similar historical cases, it infers a potential early-stage seal wear issue and generates an alert with a confidence score. 
    \ding{174}Finally, the inspector targets the belt and performs a sealing test to validate the anomaly and records the inspection outcome, which can be stored for future inspections. 
    This workflow enables the system to recognize early weak degradation signals as long-term records continuously accumulate.
\end{example}

    Existing LLM-assisted equipment inspection systems mainly adopt two paradigms: \textit{(i)} \textit{standalone-LLM reasoning}~\cite{ref5,ref29,ref30}, which directly analyzes current inspection records, and \textit{(ii) retrieval-augmented reasoning}~\cite{ref34,bib1001, bib1002, bib1010, bib1012}, which incorporates external knowledge through retrieval.
    As shown in Table~\ref{tab:pilot}, there's an inherent incompatibility between early warning accuracy and computational efficiency. Existing systems usually fail to effectively model equipment health evolution or provide reliable early warnings for the following reasons. 
    First, both paradigms implicitly formulate each inspection cycle as an independent task, thereby fail to exploit cross-cycle evidence accumulated during long-term equipment inspection.
    Second, when reasoning over heterogeneous inspection records spanning multiple components of large-scale equipment, they lack the structured historical context associating critical abnormalities with relevant evidence, leading to inaccurate fine-grained fault diagnosis. 
    Furthermore, industrial equipment continuously accumulates thousands of inspection records over its service life, in which long-horizon degradation rather than isolated observations characterizes equipment health. 
\begin{table}[t]
\centering
\small
\caption{Comparison of different inspection paradigms. 
}
\label{tab:pilot}
\vspace{-1.5em}  
\begin{tabular}{lccc}
\toprule
\textbf{Method} & \textbf{Acc.(\%) $\uparrow$} & \textbf{Token $\downarrow$} & \textbf{Time(ms) $\downarrow$} \\
\midrule
Standalone LLM~\cite{ref30} & 46.0 & 10086.4 & 18760.0 \\
Full RAG~\cite{ref34}        & 55.0 & \textbf{4496.2} & \textbf{2934.4} \\
A-MEM~\cite{ref35}          & \textbf{62.0} & 3702.1 & 8910.8 \\
\bottomrule
\end{tabular}
\vspace{-1em}  
\end{table}
    We argue that it is crucial for equipment inspection to selectively reason over long-horizon historical evidence accumulated across inspection cycles, which motivates us to take a memory-enhanced long-horizon agentic equipment inspection. 
Agentic memory equips LLM with an external memory that persists across queries and can be continuously updated and retrieved~\cite{ref11,bib1017,bib1018}, fitting equipment inspection where long-term degradation must be tracked across cycles.
    However, it faces two key challenges to implement memory-enhanced agentic inspection as follows.
    \begin{itemize}[leftmargin=1.5em, itemsep=1pt]
        \item \textit{Challenge 1}:
        \textit{
            Existing memory systems~\cite{bib1013,ref11} rely on textual continuity to segment conversation history. 
            However, inspection texts in equipment inspection are typically formed with heterogeneous functional-roles units (e.g., equipment status, safety anomalies, and maintenance actions), where generic segmentation fails to distinguish one from another, leading to degraded reasoning results over such mixed evidence.
        }
        
        \item \textit{Challenge 2}: \textit{Existing retrieval mechanisms~\cite{bib1017,bib1018} use task-agnostic relevance metrics (e.g., semantic similarity) rather than inspection-oriented contribution metrics to rank memory units, fail to prioritize critical anomalies like belt peeling over trivial records such as normal temperature.
        Consequently, it will cause unnecessary context expansion and limited reasoning gains.}


        
    \end{itemize}

To overcome these challenges, we propose a contribution-aware memory framework for LLM-assisted equipment inspection, \sysname{}, supporting a human-in-the-loop early-risk screening system.
Our framework follows a three-stage workflow tailored to the inspection domain. 
{First, we design a functional-role-aware segmentation module, which parses raw inspection logs into functional evidence units, thereby clearly separating equipment status, safety anomalies, and maintenance action.}
Second, we introduce a Shapley-based valuation module to quantify each memory unit's contribution to future reasoning, ensuring that high valued evidences are selected.
Finally, we propose a priority-based storage and retrieval module maintains a max-heap of memory units ranked by contribution, so long-term evolving components like belt wear are prioritized over static ones like lighting status.

The main contributions of this paper are summarized as follows:
\begin{itemize}[leftmargin=1.5em, itemsep=1pt]
    \item We propose a contribution-aware memory framework for LLM-assisted equipment inspection, addressing the challenges of long-term degradation tracking and early warnings.
    \item We introduce a Shapley-based valuation mechanism to quantify the contributions of memory units to diagnosis results, enabling token reduction while prioritizing critical anomalies.
    \item We conduct extensive experiments on a real-world steel equipment inspection dataset, demonstrating that our method significantly outperforms representative full-RAG and generic agentic-memory baselines, improving early warning accuracy by up to 4\% while reducing token overhead by up to 72.9\%.
\end{itemize}


The remainder of this paper is organized as follows.
We present the related work corresponding to this paper in Section 2. 
We formulate the equipment inspection task and provides the necessary notation in Section 3. 
Detailed specifics of our ConMem system are introduced in Section 4 and the experimental results are reported in Section 5.
Finally, we conclude this work in Section 6.

\section{RELATED WORK} \label{sec:related}

\noindent\textbf{LLM for Industrial Equipment Inspection.}
{Recent years have witnessed growing interest in applying large language models (LLMs) to industrial equipment inspection.
Researchers have explored various LLM-based approaches to assist inspectors in analyzing routine inspection logs, identifying surface anomalies such as belt peeling and roller wear, and generating maintenance recommendations.
A common paradigm is retrieval-augmented generation (RAG), where an LLM queries historical inspection logs to retrieve relevant records before generating responses.
RAG-based systems have been developed to retrieve past anomaly records and maintenance actions for steel rolling equipment and conveyor systems~\cite{bib1001, bib1002, bib1010, bib1012}.
Several studies have applied semantic segmentation and object detection to identify equipment defects such as cracks and surface wear in steel production lines~\cite{bib1003, bib1004, bib1007, bib1008, bib1009}.
Besides, in the steel manufacturing domain, researchers have also focused on predictive maintenance and degradation tracking for critical components like blast furnace shells and continuous casting rollers~\cite{bib1005, bib1006, bib1011}.
However, these RAG-based approaches retrieve the full history or fixed-length chunks for each query, leading to high token overhead and slow retrieval speed~\cite{bib1013, bib1014}. Moreover, they treat each query independently and cannot capture long-term degradation trends across cycles.
In our work, we propose an agentic memory tailored for industrial equipment inspection that persistently store and update the inspection logs across cycles.}

\noindent\textbf{LLM Memory Systems.}
{External memory management in LLM agents has been explored across multiple domains, yet each line of work carries assumptions that do not align with industrial equipment inspection. Early studies focus on conversational memory, storing user-LLM dialogue histories as textual records. SCM~\cite{bib1013} and Generative Agents~\cite{ref11} maintain social or dialogue contexts through hierarchical memory and memory streams, respectively. MemGPT~\cite{ref10} and MemoChat~\cite{ref4} target extended conversations and topic-level summarization. Other works focus on task-oriented memory, such as MAC~\cite{bib1017} compressing document information and AWM~\cite{bib1018} extracting reusable workflows from agent trajectories. More recently, Mem0~\cite{ref38} and Zep~\cite{ref13} provide lightweight or persistent memory layers for general LLM applications. While effective in general-purpose chat or tool-use scenarios, these methods assume that memory units are either equally important or selected heuristically. However, equipment inspection logs contain heterogeneous entries, whose long-term value varies greatly.
Moreover, inspection queries span hundreds of cycles, requiring the system to track slow degradation trends while filtering out routine records.
To address these issues, we focus on the equipment inspection scenario, explicitly modeling inspection record segments and evaluating memory contribution for effective early warning.

\section{PRELIMINARIES}
\subsection{Problem definition}

\begin{definition}[Inspection Records]

Formally, an inspection record is represented as
$I = (e, c, \textit{R}, t, o),$
where $e$ denotes the equipment ID, $c$ is the inspected component (e.g., belt, fire extinguisher, lighting), $t$ is the inspection timestamp, $o$ is the human check result, and $\mathcal{R}=\{r_1,\ldots,r_n\}$ is a collection of inspection entries. Each entry is represented as $r=(\tau,w), \tau \in \{\textit{Parameter Status},\allowbreak\textit{Safety Anomalies}, \allowbreak\textit{Maintenance Advice}\}$, denotes the information category, and $w$ denotes the corresponding inspection content.

\end{definition}
A steel inspection record corresponding to a single inspection task performed on a specific equipment component. It aggregates multiple categories of inspection information, including equipment parameter status, safety anomaly analysis, and maintenance action. 

\begin{definition}[Memory Unit]

A memory unit corresponds to a semi-structured inspection record generated from a single inspection task. It is represented as
$m=(e,c,C,t),$
where $e$ denotes the equipment ID, $c$ is the inspected component, $t$ is the inspection timestamp, $C$ denotes the inspection content.

\end{definition}
\begin{definition}[Evidence Unit]
The inspection content is semantically organized into a set of functional-role-aware evidence units
$C=\{v^{S},v^{A},v^{M}\},$
where $v^{S}$, $v^{A}$, $v^{M}$ respectively correspond to the inspection segment describing the equipment status, safety anomalies, and maintenance actions. 

\end{definition}

\begin{definition}[Human-in-the-Loop Early-Warning Inspection]
\label{def:early_warning}

A memory-enhanced steel inspection problem is defined as a tuple
$(\mathcal{D}, q_t, \mathcal{M}_t, \mathrm{LLM})$,
where $q_t$ denotes a routine inspection query issued at time $t$, and
$\mathcal{D}_{t}^{W}
=
\{d_i \in \mathcal{D} \mid t-W \leq t_i \leq t\}$
denotes the inspection records available within a look-back window
$W$.
The system maintains a persistent memory store $\mathcal{M}_t$,
which is updated across inspection cycles using only records available
up to time $t$.
For each query $q_t$, it retrieves a relevant subset
$\mathcal{M}_{q_t} \subseteq \mathcal{M}_t$ combining with the
current inspection context $\mathcal{X}_{q_t}$ to generate
$a_t = (\hat{y}_t, \mathcal{E}_t, r_t) = 
\mathrm{LLM}\left(q_t,
\mathcal{M}_{q_t},
\mathcal{X}_{q_t}
\right)$,where $\hat{y}_t$ is an early-risk indicator,
$\mathcal{E}_t$ contains the supporting historical evidence, and
$r_t$ provides the corresponding risk explanation and recommended
inspection target.

Let $g_j=(e_j,c_j,f_j,t_j^c)$ denote a subsequently confirmed
anomaly equipment $e_j$, component $c_j$, fault type
$f_j$, and confirmation time $t_j^c$. We write $g_j\sim q_t$ when
$g_j$ and $q_t$ refer to the same equipment and component, and their
fault types satisfy the matching criterion
$f_j\simeq f_{q_t}$. Given a follow-up horizon $h_d$, the
early-warning label is defined as
$y_t=\mathbb{I}[\exists\,g_j\sim q_t:
t_j^c-t\in(0,h_d)]$.
Importantly, the confirmed result becomes part of the historical
evidence available only to subsequent routine inspection queries
issued after $t_j^c$.

The objective is to maximize the recall of subsequently confirmed anomalies while maintaining an acceptable false-alert rate and providing sufficient lead time for targeted inspection.
The generated warning is therefore a decision-support signal rather than an autonomous fault diagnosis: professional inspectors remain responsible for confirming anomalies and determining corresponding maintenance actions.
\end{definition}


\begin{figure*}[!htb]
\vspace{-1em}
\centering
\includegraphics[width=\textwidth]{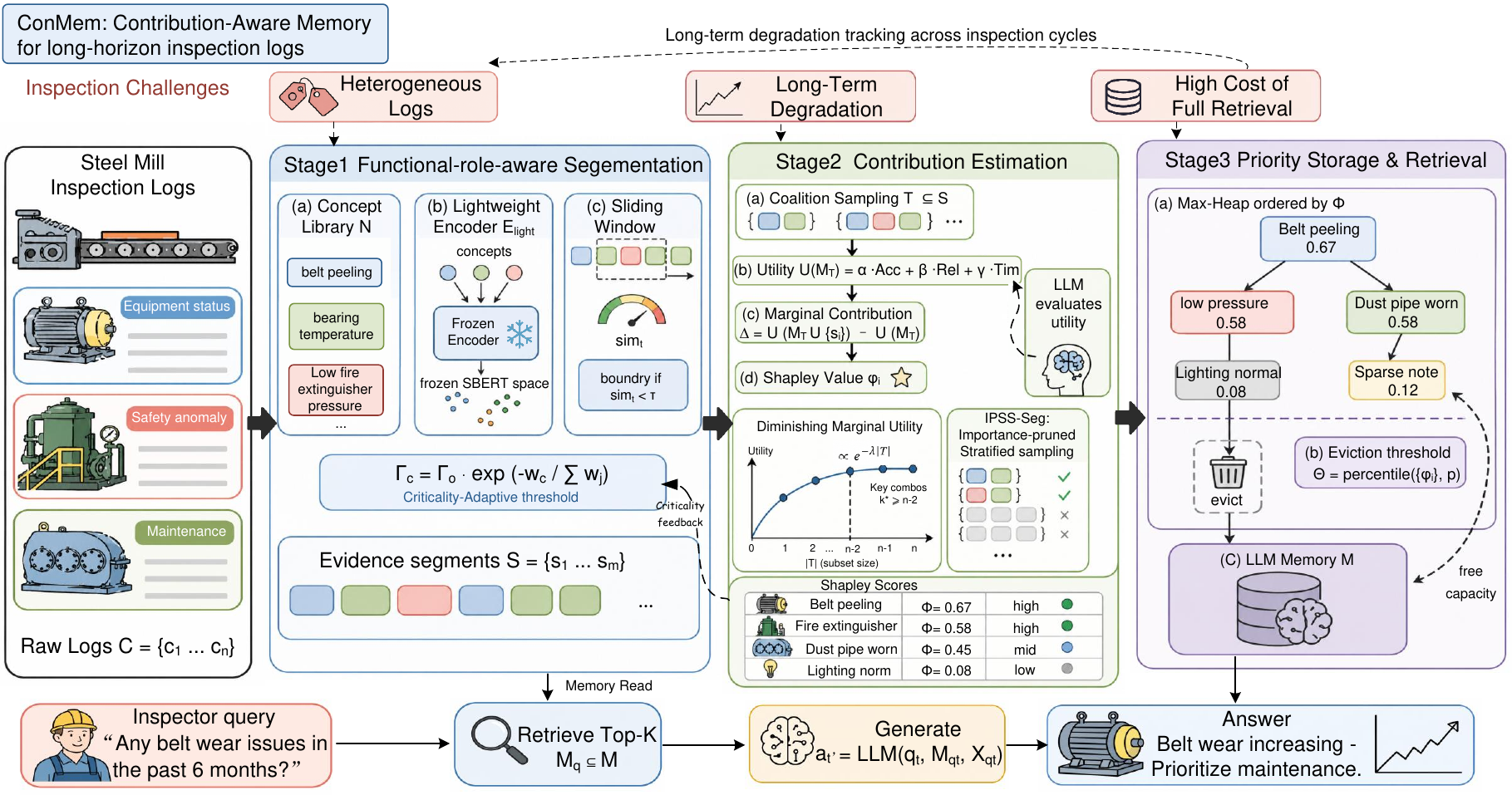}
\caption{Overall framework of ConMem.}
\label{fig:framework}
\vspace{-1.5em}
\end{figure*}

\section{METHOD}
\subsection{ConMem Overview}

This section details the technical implementation of the three modules introduced in Section 4, with all techniques specifically adapted for steel equipment inspection scenarios.

To address the limitations of existing historical records segmentation and memory storage mechanisms for long-horizon equipment inspection, which often lead to inefficient retrieval and reduced early-warning accuracy, we propose a \underline{con}tribution-aware \underline{mem}ory framework for LLM-assisted equipment inspection, supporting a human-in-the-loop early-risk screening system. As shown in \figref{fig:framework}, ConMem consists of three key components.
(1) The functional-role-aware segmentation module follows a standard NLP pipeline to decompose each inspection record into several functional roles, enabling fine-grained valuation of evidence units (in \secref{sec:segmentation}). 
(2) The contribution estimation module estimates the contribution of each memory unit to anomaly diagnosis through a utility function. (in \secref{sec:Contribution}).
(3) The priority-based efficient storage module supports efficient long-horizon evidence retrieval under constrained memory capacity through a priority-based storage strategy that dynamically retains high-value memory units while evicting low-value ones during memory updates (in \secref{sec:Storage}).

\subsection{Functional-Role-Aware Segmentation
}
\label{sec:segmentation}
For equipment inspection, logs are inherently heterogeneous. A single record may contain equipment status entries, safety anomalies, and maintenance records. A query "are there any belt wear issues in the past six months?" may get more evidence information from the equipment status entries like "Vibration slightly exceeds 4.5 mm/s" than from maintenance records entries like "Replace all bearings". 
Generic segmentation methods~\cite{bib1013,ref11} cannot distinguish different roles, thus, it demands a segmentation approach that respects the functional role within inspection records.

Our observation is that steel inspection logs naturally contain functional-role shifts. For instance, a log sequence often transitions from equipment status reporting to safety anomaly detection, then to maintenance action recording. 
Based on this observation, we design a functional-role-aware segmentation module that parses raw inspection logs into functional-role-aware evidence units.

Specifically, given available historical inspection logs $\mathcal{C} = \{c_1, \ldots, c_n\}$, we first extract noun phrases $\mathcal{N}$ using standard NLP pipelines.
A lightweight encoder $E_{\text{light}}$ is trained to map these concepts to a frozen sentence encoder space (e.g., SBERT) by minimizing:
\[
\min_{\theta} \sum_{n \in \mathcal{N}} \| E_{\text{light}}(n) - E_{\text{frozen}}(n) \|^2. \tag{1}
\]
Then after $E_{\text{light}}$ is prepared, the segmentation is processed on dynamic sliding windows. For sliding window position $s_t$, concept similarity is computed as:
\[
\text{sim}_{st} = \frac{1}{|\mathcal{N}_{st}|} \sum_{n \in \mathcal{N}_{st}} \max_{c \in \mathcal{C}_{\text{prev}}} \cos(E_{\text{light}}(n), E_{\text{light}}(c)). \tag{2}
\]
A segment boundary is detected when $\text{sim}_{st} < \tau$, indicating a role shift.
For example, transitioning from equipment status to safety anomaly.
Besides, to further optimize segmentation quality for steel inspection, we introduce a domain-adaptive threshold adjustment mechanism. Finally, given the varying importance of different equipment components, the threshold $\tau$ is dynamically calibrated based on component criticality $\medop{\tau_c = \tau_0 \cdot \exp(-{w_c}/{\sum_{j} w_j})}$,
where $w_c$ denotes the criticality weight of component $c$ (e.g., conveyor belts receive higher weight than lighting fixtures). Thus, this ensures that segments containing critical components like belt drives are preserved with higher fidelity.

Consequently, each evidence unit preserves complete evidence within a single functional role, providing reliable memory entries for fine-grained valuation and efficient long-horizon retrieval.

\subsection{Contribution Estimation for Inspection Logs}
\label{sec:Contribution}
Not all evidence units are equally valuable for future queries, each memory unit obtained from \secref{sec:segmentation} lacks contribution estimation. A routine temperature reading from three months ago contributes little to answering ``Has the belt shown signs of wear?'' whereas a single occurrence of ``fire extinguisher underpressure'' may be critical.
In equipment inspection, the long-term value of a record depends on two factors: (i) its relevance to equipment health (anomalies matter more than routine status), and (ii) its timeliness (older records gradually lose value).
However, existing agentic-memory methods lack a principled way to quantify such contributions, leading to critical anomalies being evicted in favor of trivial records.

Our key observation, validated on steel inspection data, is that marginal utility diminishes as more context is provided. After retrieving the top 3 most valuable segments, adding more inspection records yields negligible accuracy improvement. This diminishing returns property motivates a valuation mechanism that can fairly rank segments by their contribution.

Thus, we adopt the marginal contribution underlying the Shapley value from cooperative game theory~\cite{ref31,ref32,ref33}.
Specifically, for $m$ segments $\mathcal{S} = \{s_1, \ldots, s_m\}$ extracted from inspection logs, the marginal contribution-based Shapley value for segment $s_i$ is:
\[
\phi_i = \sum_{T \subseteq \mathcal{S} \setminus \{s_i\}} \frac{U(M_{T \cup \{s_i\}}) - U(M_T)}{m \cdot \binom{m-1}{|T|}}, \tag{4}
\]
where $U(\cdot)$ measures LLM response quality for representative queries, such as ``Has the belt shown degradation over the past three months?''. However, direct computation requires $O(2^m)$ evaluations. Thus, to make Shapley computation tractable for inspection scenarios, first we define the utility function $U(\cdot)$. Importantly, $U$ is estimated exclusively from chronologically earlier training data and remains fixed throughout test-time evaluation. Specifically, for each training query $q_t$ issued at time $t$, its input context and every evaluated memory coalition $M_{qt}$ contain only inspection records whose timestamps fall within the look-back window [t-W,t]. The query is constructed from the scheduled inspection task and metadata available at time $t$. The label $y_t$ is determined retrospectively from human-confirmed outcomes occurring within the follow-up window $(t,t+h_d]$. For a memory coalition $M_{qt}$ we compute
\[
U(M_{qt})=\alpha\,\mathrm{Acc}_{\mathcal Q_{\mathrm{train}}}(M_{qt})
+\beta\,\mathrm{Rel}_{\mathcal Q_{\mathrm{train}}}(M_{qt})
+\gamma\,\mathrm{Tim}(M_{qt}), \tag{5}
\]
where $\text{Acc}(M_{qt})$ measures factual accuracy against ground-truth inspection records, $\text{Rel}(M_{qt})$ captures relevance to the current equipment component, and $\text{Tim}(M_{qt})$ penalizes outdated information (e.g., records older than 12 months receive lower weight).
Then we formalize the diminishing returns property for inspection segments,
\[
\frac{\partial U(M_{T \cup \{s_i\}})}{\partial |T|} \propto e^{-\lambda |T|}, \quad \lambda > 0, \tag{6}
\]
where $\lambda$ is the decay rate empirically estimated from inspection data. 
Finally, based on this property, we use importance-pruned stratified sampling to prune low-impact combinations while preserving accuracy, reducing the number of utility evaluations to only a small subset of the original $O(2^m)$ coalition space.

\begin{table}[t]
\centering
\caption{Complexity analysis of the proposed pruned strategy. K denotes the number of elements in a coalition.}
\label{tab:complexity_analysis}
\small
\vspace{-1.2em}
\begin{tabular}{lccc}
\toprule
Strategy(n=8) & Evaluation-times $\downarrow$ & Spearman $\uparrow$ & Top-3 Overlap$\uparrow$ \\
\midrule
Full  & 256 & 1.000 & 1.000 \\
$K \geq n-3$  & 93 & 1.000 & 1.000\\
$K \geq n-2$  & \textbf{37} & \textbf{1.000} & \textbf{1.000} \\
$K \geq n-1$ & 9 & 0.997 & 0.959 \\
\bottomrule
\end{tabular}
\vspace{-1.2em}
\end{table}

\noindent\textbf{Complexity Analysis.} To validate our pruned strategy, we compare the estimated Shapley rankings under different coalition truncations. As shown in Table~\ref{tab:complexity_analysis}, the relative ranking of evidence contributions remains stable in near-complete coalitions. So, Our diminishing-marginal-utility observation and high-order truncation characterize two complementary aspects of evidence contribution. The former describes the absolute marginal gain of adding a segment decreasing as more evidences including in the context. The latter offers an efficient way to estimate shapley values.

\subsection{Priority-Based Efficient Storage}
\label{sec:Storage}
In a steel mill, the LLM's context window is a critical resource. When memory capacity is constrained, the system must decide which inspection records to retain and which to evict. A key observation is that inspection records exhibit a long-tail distribution: a small number of critical anomalies (e.g., belt peeling, fire extinguisher underpressure) are queried frequently, while the majority of routine status are rarely accessed. This observation suggests a memory-saving strategy that prioritizes high-value anomalies.

Specifically, our priority-based storage module directly leverages the Shapley values computed in the previous stage. Segments are stored in a max-heap ordered by $\hat{\phi}_i$. When memory capacity is constrained, the lowest-contribution segments are evicted first. This ensures that critical anomalies are never lost. 

For a new query $q$ such as ``Have any belt wear issues in the past six months?'', first, it retrieves top-$K$ segments semantically relevant to $q$ from high-shapley-value candidates. Then, to balance memory retention and retrieval efficiency, we define an eviction threshold $\theta$ that adapts to current memory pressure:
\[
\theta = \text{percentile}(\{\hat{\phi}_i\}_{i=1}^{m}, p), \tag{7}
\]
where $p$ is the target retention percentile (e.g., $p=20\%$ retains the top 20\% of segments by contribution). Finally, all segments with $\hat{\phi}_i < \theta$ are candidates for eviction when the context window is full. This adaptive strategy enables the agent to maintain a compact yet informative memory store.

\noindent\textbf{Effectiveness Analysis.} We evaluate the proposed storage strategy by varying the proportion of Shapley-ranked evidence retained in memory. As shown in Table~\ref{tab:retention_ratio}, retaining only the top 25\%-50\% of Shapley-ranked evidence achieves superior accuracy, indicating that the proposed priority-based storage maintains a compact memory for efficient retrieval.

\begin{table}[t]
\centering
\caption{Analysis result of different Shapley-ranked ratio.}
\label{tab:retention_ratio}
\vspace{-1.2em}
\begin{tabular}{lcccccc}
\toprule
Strategy & 10\% & 25\% & 50\% & 75\% & 90\% & 100\% \\
\midrule
Acc.(\%) & 51 & 76 & 72 & 57 & 53 & 52 \\
\bottomrule
\end{tabular}
\vspace{-1.2em}
\end{table}

\begin{algorithm}[!htb]
\caption{ConMem for equipment inspection}
\vspace{-0.2em}
\label{alg:ConMem}
\begin{algorithmic}[1]
\REQUIRE Inspection logs $\mathcal{D}$, query $q_t$, inspection prompt $X_{qt}$, budget $\rho$
\ENSURE Response $a_{t'}$, retrieved memory $\mathcal{M}_{qt}$

\STATE \textbf{Stage 1: Functional-Role-Aware Segmentation}
\STATE Extract noun phrases (e.g., ``belt peeling'', ``fire extinguisher'') from $\mathcal{D}$;
\STATE Detect boundaries via $\text{sim}_t < \tau$, output segments $\mathcal{S}$;

\STATE \textbf{Stage 2: Shapley Valuation}
\STATE Compute $\phi_i = \sum_{T} \frac{U(M_{T\cup\{s_i\}}) - U(M_{T})}{m \cdot \binom{m-1}{|T|}}$ for each $s_i \in \mathcal{S}$;

\STATE \textbf{Stage 3: Priority Storage \& Retrieval}
\STATE Store segments in max-heap ordered by $\hat{\phi}_i$;
\STATE Retrieve top-$K$ segments $\mathcal{M}_{qt}$ relevant to $q_t$;
\STATE Generate $a = \mathrm{LLM}({q_t}, \mathcal{M}_{qt},X_{qt})$;

\RETURN $a_{t'}$, $\mathcal{M}_{qt}$

\end{algorithmic}
\vspace{-0.2em}
\end{algorithm}

\subsection{Complexity Analysis of ConMem}

Algorithm~\ref{alg:ConMem} presents the pseudo-code of our ConMem framework for steel equipment inspection. The sampling algorithm requires $O(\rho)$ LLM inferences, where $\rho$ is the sampling budget. The relationship between sampling budget $\rho$ and the minimum evaluated combination size $k^*$ is given by, $\rho = \sum_{j=k*}^{n} \binom{m}{j}$, with typical settings where $k^* \ge n-2$ and $m \le 20$, which is common in steel inspection where approximately 20 anomaly segments are extracted from each batch of inspection logs, This reduces the number of utility evaluations from exponential $O(2^n)$ to quadratic $O(n^2)$  while maintaining approximation accuracy within $1\%$.
Thus, the overall time complexity of our method for a single inspection query is:
\[
T_{\text{total}} = T_{\text{seg}} + T_{\text{val}} + T_{\text{ret}} = O(n \cdot d) + O(\rho \cdot \tau) + O(K \log m), \tag{9}
\]
where $T_{\text{seg}}$ is the segmentation cost, linear in log length $n$ and embedding dimension $d$. $T_{\text{val}}$ is the Shapley valuation cost, requiring $\rho$ LLM calls, and $T_{\text{ret}}$ is the retrieval cost, logarithmic in memory size $m$. For steel equipment inspection, this translates to sub-second response times, reducing from Naive Context retrieval's 18.76 seconds to about 2.5 seconds, thereby enabling in time anomaly detection.

\section{EXPERIMENTS}

In this section, we conduct comprehensive experiments to evaluate the proposed ConMem framework on a steel equipment inspection dataset. We aim to answer the following research questions:
\begin{itemize}[leftmargin=1.5em, itemsep=1pt]
    \setlength{\topsep}{-0.5em}       
    \setlength{\itemsep}{0pt}      
    \setlength{\parsep}{0pt}
    \setlength{\partopsep}{0pt}

    \item \textbf{RQ1:} Can ConMem achieve comparable or superior accuracy to Full RAG while using significantly fewer tokens?
    \item \textbf{RQ2:} How does ConMem compare against heuristic memory selection strategies in terms of retrieval efficiency?
    \item \textbf{RQ3:} What is the impact of the diminishing marginal utility phenomenon on memory retention?
\end{itemize}



\begin{table*}[t]
\centering

\captionsetup{skip=2pt}

\caption{\textbf{Comprehensive performance comparison on the steel equipment inspection QA task.}}
\label{tab:comprehensive_main_results}

\scriptsize
\setlength{\tabcolsep}{3pt}

\renewcommand{\arraystretch}{0.98}

\setlength{\aboverulesep}{0.25ex}
\setlength{\belowrulesep}{0.25ex}

\resizebox{\textwidth}{!}{
\begin{tabular}{llc*{6}{c}}
\toprule

\multirow{3}{*}{\textbf{Family}}
& \multirow{3}{*}{\textbf{Method}}
& \multirow{3}{*}{\textbf{Memory Budget}}
& \multicolumn{1}{c}{\textbf{QA Metrics}}
& \multicolumn{4}{c}{\textbf{Efficiency}}
& \multirow{3}{*}{\textbf{Acc./1K}}
\\

\cmidrule(lr){4-4}
\cmidrule(lr){5-8}

& &
& \textbf{Acc.}
& \textbf{Input Tok.}
& \textbf{Tok. Red.}
& \textbf{Time}
& \textbf{Time Red.}
&
\\

& &
& \textbf{(\%) $\uparrow$}
&
& \textbf{(\%) $\uparrow$}
& \textbf{(ms) $\downarrow$}
& \textbf{(\%) $\uparrow$}
& $\uparrow$
\\
\midrule


\multirow{3}{*}{\makecell[l]{Context /\\retrieval}}

& Naive Context-8K
& Full
& 46.0
& 10086.4
& 0.0
& 18760.0
& 0.0
& 4.56
\\

& Full RAG
& Top-10
& 55.0
& 4496.2
& 55.4
& 2934.4
& 84.4
& 12.23
\\

& Full Event Retrieval
& Full
& 72.0
& 2251.1
& 77.7
& 3837.4
& 79.5
& 31.98
\\

\midrule


\multirow{6}{*}{\makecell[l]{Heuristic\\memory}}

& \multirow{3}{*}{Random Memory}
& 10\%
& 31.0
& 1193.3
& 88.2
& 2020.9
& 89.2
& 25.98
\\

&
& 25\%
& 50.0
& 1191.7
& 88.2
& 2364.9
& 87.4
& 41.96
\\

&
& 50\%
& 65.0
& 1189.0
& 88.2
& 2340.3
& 87.5
& 54.67
\\

&
\multirow{3}{*}{Rule-based Memory}
& 10\%
& 40.0
& 1199.5
& 88.1
& 2206.3
& 88.2
& 33.35
\\

&
& 25\%
& 52.0
& 1200.5
& 88.1
& 2403.3
& 87.2
& 43.32
\\

&
& 50\%
& 64.0
& 1196.9
& 88.1
& 2414.7
& 87.1
& 53.47
\\

\midrule


\multirow{4}{*}{\makecell[l]{External\\agent memory}}

& A-MEM~\cite{ref35}
& Top-20
& 62.0
& 3702.1
& 63.3
& 8910.8
& 52.5
& 16.75
\\

& MemoRAG~\cite{ref36}
& Top-20
& 72.0
& 4404.2
& 56.3
& 9977.9
& 46.8
& 16.35
\\

& ActMem~\cite{ref37}
& Top-20
& 73.4
& 1176.5
& 88.3
& 10507.3
& 44.1
& 62.39
\\

& Mem0~\cite{ref38}
& Top-20
& \underline{74.7}
& 1284.3
& 87.3
& 2750.5
& 85.3
& 58.16
\\

\midrule


\multirow{5}{*}{\makecell[l]{ConMem /\\Ablation}}

& \multirow{3}{*}{ConMem}
& 10\%
& 51.0
& 1191.5
& 88.2
& 2031.6
& 89.2
& 42.80
\\

&
& 25\%
& \cellcolor{gray!15}\textbf{76.0}
& \cellcolor{gray!15}1192.1
& \cellcolor{gray!15}88.2
& \cellcolor{gray!15}2510.4
& \cellcolor{gray!15}86.6
& \cellcolor{gray!15}\textbf{63.75}
\\

&
& 50\%
& 72.0
& 1192.8
& 88.2
& 3099.2
& 83.5
& 60.36
\\

&
\multirow{2}{*}{ConMem Ablation}
& Record-level SSV
& 32.0
& 1088.0
& 89.2
& 1811.2
& 90.3
& 29.41
\\

&
& No-pruning Shapley
& 52.0
& 1336.8
& 86.7
& 2780.0
& 85.2
& 38.90
\\

\bottomrule
\end{tabular}
}

\vspace{2pt}

\begin{minipage}{0.98\textwidth}
\footnotesize
\textit{Notes.}
Tok. Red. denotes token reduction relative to the average input tokens of Naive Context-8K.
Acc./1K denotes answer accuracy divided by average input tokens and multiplied by 1000.
Bold and underlined entries indicate the best and second-best accuracy values.
\end{minipage}

\end{table*}
\subsection{Experimental Setup}

\textbf{Dataset.} 
From real-world steel equipment inspection practices, we construct a large-scale inspection dataset based on authentic inspection records.
The dataset contains approximately 30,080 inspection entries, each preserving key fields: equipment ID, component, timestamp, responsible personnel, and anomaly description. From the anomaly descriptions, we automatically extract anomaly event cards (e.g., ``belt back peeling severe'', ``fire extinguisher underpressure'', ``dust pipe worn'') to build the memory candidate pool. Gold labels are hidden as standard answers for evidence-level validation.

\noindent\textbf{Evaluation Metrics.} We adopt accuracy (recall of subsequently confirmed anomalies), Token Overhead, and Response Time (ms).

\noindent\textbf{Implementation Details.} Experiments are conducted on a server equipped with an NVIDIA GeForce RTX 5090 GPU and 128GB main memory. The LLM backbone is GPT-4o~\cite{ref15}. We use the all-MiniLM-L6-v2~\cite{ref18} as the frozen sentence encoder. For ConMem, we set the sampling budget $\rho$ according to minimum combination size $k^*$, and $k^* = n-2$. For records look-back window, we set W=366 days. For human reinspection horizon, we set $h_d$=14 days. For Full RAG, the context window is set to 8K tokens. The coefficients $\alpha =0.45, \beta =0.45, \gamma=0.1$ defined in the utility function $U(\cdot)$ are tuned specifically for steel inspection tasks. We construct queries based on inspection logs and partition them into training and test sets by year. The training of the lightweight encoder $E_{\text{light}}$ and the evaluation of the utility function $U(\cdot)$ are conducted on the training set, and experiments are performed on the test set. 

\noindent\textbf{Baselines.}
We compare ConMem with representative baselines from four categories, including several state-of-art methods.

\noindent\textbf{(i) Context-based retrieval methods.}
These methods directly provide raw historical inspection logs to the LLM.
\begin{itemize}[leftmargin=0.8em, itemsep=2pt]
    \item \textbf{Naive Context-8K}: It directly uses the most recent records within an 8K-token context window through TF-IDF retrieval and greedy 8K packing strategy.
    \item \textbf{Full RAG}: It represents methods that retrieve relevant inspection records from the complete history without memory construction. We implement it using a two-stage retrieval pipeline, where TF-IDF retrieves top-100 candidates and E5 reranking selects the final top-10 records.
    \item \textbf{Full Event Retrieval}: It directly provides complete structured inspection events (i.e., event cards) from the entire historical records to the LLM without memory compression or pruning, serving as an upper-bound retrieval setting with full access to historical information.
\end{itemize}
\noindent\textbf{(ii) Heuristic Memory methods.}
These methods construct memories using heuristic rules instead of value-aware selection.
\begin{itemize}[leftmargin=0.8em, itemsep=2pt]
    \item \textbf{Random Memory}: It randomly selects memory segments from historical records under the same memory budget as ConMem.
    \item \textbf{Rule-based Memory}: It selects memory segments according to heuristic criteria, including recency and access frequency.
\end{itemize}
\noindent\textbf{(iii) External agent memory methods.}
We further compare ConMem with representative memory-augmented LLM frameworks that maintain and retrieve long-term memories.
\begin{itemize} [leftmargin=0.8em, itemsep=2pt]
    \item \textbf{A-MEM~\cite{ref35}}: It constructs structured memories from historical interactions and performs memory retrieval through its official memory organization pipeline.
    \item \textbf{MemoRAG~\cite{ref36}}: It combines retrieval-augmented generation with memory clues to enable efficient long-horizon access.
    \item \textbf{ActMem~\cite{ref37}}: It constructs a structured causal–semantic memory graph from atomic facts extracted from historical interactions, and performs retrieval.
    \item \textbf{Mem0~\cite{ref38}}: It provides a memory layer that extracts, stores, and retrieves persistent memories via an LLM-driven extraction pipeline with semantic retrieval over a vector database.
\end{itemize}
\noindent\textbf{(iv) ConMem variants.}
We perform ablations to examine the impact of fine-grained memory valuation and pruning.
\begin{itemize}[leftmargin=0.8em, itemsep=2pt]
    \item \textbf{Record-level SSV}: It applies the proposed Shapley-based valuation to raw inspection records rather than fine-grained evidence units, assessing the effect of memory granularity.
    \item \textbf{No-pruning Shapley}:It keeps all memory units and uses Shapley values only for valuation, removing memory pruning.
\end{itemize}

\subsection{Main Results}

Table~\ref{tab:comprehensive_main_results} reports the overall performance of all compared methods.
\textbf{For RQ1}, ConMem achieves 76.0\% accuracy, outperforming Mem0~\cite{ref38} 74.7\%, while reducing token overhead by 88.2\% and response time by 86.6\% compared to Full RAG.
\textbf{For RQ2}, ConMem exhibits the highest accuracy for high-value anomalies and the lowest time and token consumption, confirming that the priority-based retention can effectively preserve critical information.
Besides, \figref{fig:seasonal_all_methods} further confirms the robustness of ConMem.

Although Mem0~\cite{ref38} achieves accuracy and efficiency comparable to ConMem when GPT-4o is used as the backbone, this result should be interpreted together with its sensitivity to backbone capability. When GLM-4~\cite{ref17} is used as the backbone, the accuracy of Mem0 decreases substantially to 45.9\%, revealing a strong dependence on the LLM's implicit memory-management capability. That's because Mem0 does not explicitly design a mechanism for retrieval or storage. Its memory extraction, consolidation, and update decisions are largely delegated to the backbone LLM. In contrast, even under the GLM-4 LLM, the accuracy of ConMem remains 76\%. Since its effectiveness is derived not only from the reasoning capability of the backbone, but also from a model-independent memory-management mechanism. This distinction is particularly important for industrial deployment, where reliance on expensive proprietary LLM models may be impractical and the memory system must remain effective across heterogeneous backbone models.

\subsection{Ablation Study}
Table~\ref{tab:ablation} presents ablation results to isolate the contribution of each ConMem component. \textbf{Addressing RQ3}, the 25\% ConMem model achieves the best performance. Removing functional-role-aware segmentation degrades accuracy to 32.0\%, confirming its critical role in helping fine-grained evidence valuation. Removing Shapley valuation reduces accuracy to 52.0\%, while removing importance pruning yields a moderate drop to 66.0\%. These results validate that the diminishing marginal utility phenomenon is correctly captured by our Shapley-based valuation.
\begin{table}[!htb]
\centering
\vspace{-1em}  
\caption{Ablation study of ConMem components.}
\label{tab:ablation}
\small
\vspace{-1em}  
\begin{tabular}{lc}
\hline
\textbf{Configuration} & \textbf{Acc.} \\
\hline
ConMem (25\%) & 76.0 \\
w/o functional-role-aware seg. & 32.0 \\
w/o Shapley valuation & 52.0 \\
w/o importance pruning & 66.0 \\
\hline
\end{tabular}
\vspace{-1.5em} 
\end{table}

\begin{figure}[thb]
    \centering
    \includegraphics[width=\linewidth]{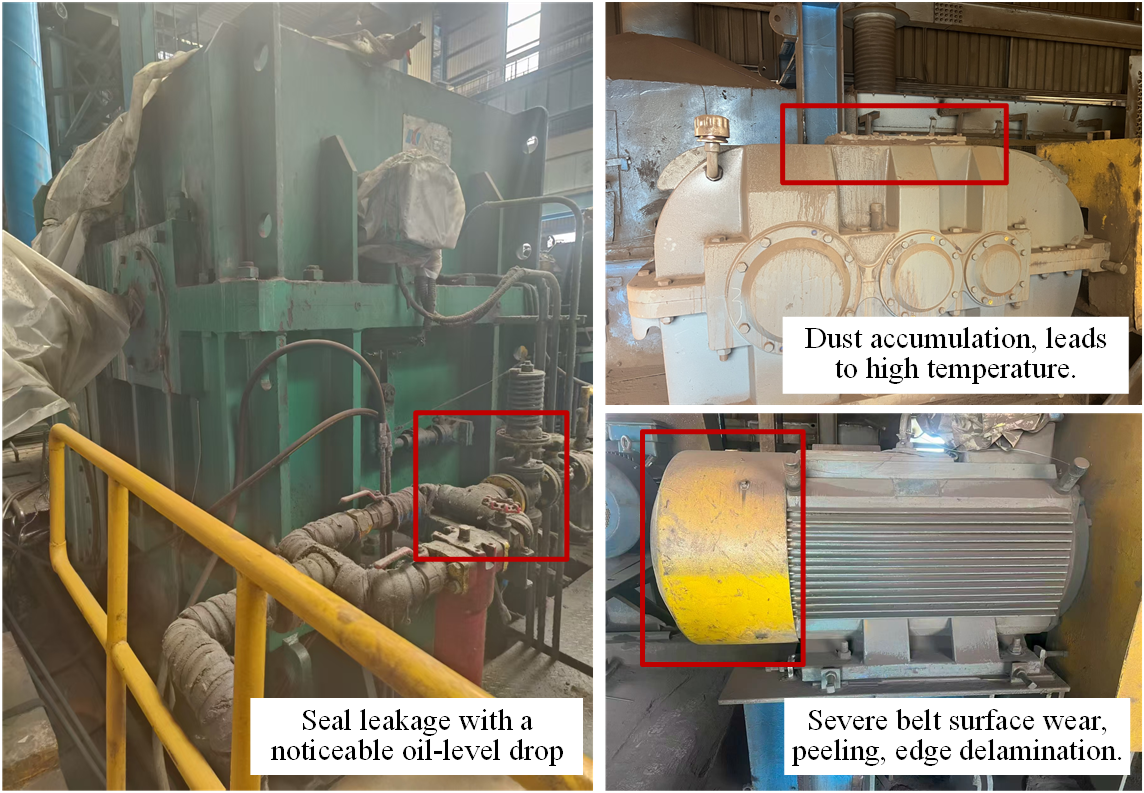}
    \caption{Deployment Cases: On-site inspection.}
    \label{fig:cases}
    \vspace{-2em}
\end{figure}

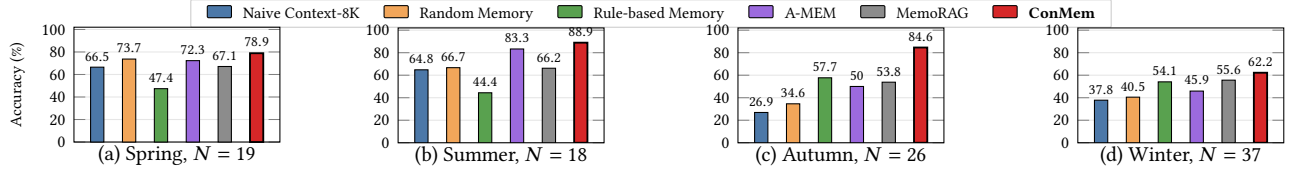
\begin{figure*}[t]
\centering
\vspace{-1em}

\begin{tikzpicture}
\node[
    draw=gray!40,
    fill=white,
    rounded corners=0.5pt,
    inner xsep=6pt,
    inner ysep=2pt
]{
\scriptsize
\tikz \filldraw[draw=black,fill=cNaive] (0,0) rectangle (0.32,0.16); \hspace{0.2em} Naive Context-8K
\hspace{1em}
\tikz \filldraw[draw=black,fill=cRandom] (0,0) rectangle (0.32,0.16); \hspace{0.2em} Random Memory
\hspace{1.0em}
\tikz \filldraw[draw=black,fill=cRule] (0,0) rectangle (0.32,0.16); \hspace{0.2em} Rule-based Memory
\hspace{1.0em}
\tikz \filldraw[draw=black,fill=cAMEM] (0,0) rectangle (0.32,0.16); \hspace{0.2em} A-MEM
\hspace{1.0em}
\tikz \filldraw[draw=black,fill=cMemoRAG] (0,0) rectangle (0.32,0.16); \hspace{0.2em} MemoRAG
\hspace{1.0em}
\tikz \filldraw[draw=black,fill=cConMem] (0,0) rectangle (0.32,0.16); \hspace{0.2em} \textbf{ConMem}
};
\end{tikzpicture}

\vspace{-0.6em}

\begin{minipage}[t]{0.242\textwidth}
\centering
\begin{tikzpicture}
\begin{axis}[
    ybar,
    width=\linewidth,
    height=0.72\linewidth,
    ymin=0,
    ymax=102,
    xmin=0.35,
    xmax=1.65,
    bar width=5pt,
    ylabel={Accuracy (\%)},
    ylabel style={font=\scriptsize},
    y tick label style={font=\scriptsize},
    xtick=\empty,
    xlabel={(a) Spring, $N=19$},
    xlabel style={font=\small, yshift=2pt},
    ytick={0,20,40,60,80,100},
    ymajorgrids=true,
    grid style={gray!20},
    axis line style={black},
    tick align=inside,
    nodes near coords,
    nodes near coords style={
        font=\tiny,
        rotate=0,
        anchor=south,
        yshift=-0.8pt
    }
]

\addplot[draw=black, fill=cNaive,   bar shift=-30pt] coordinates {(1,66.5)};
\addplot[draw=black, fill=cRandom,  bar shift=-18pt]  coordinates {(1,73.7)};
\addplot[draw=black, fill=cRule,    bar shift=-6pt]  coordinates {(1,47.4)};
\addplot[draw=black, fill=cAMEM,    bar shift=6pt]   coordinates {(1,72.3)};
\addplot[draw=black, fill=cMemoRAG, bar shift=18pt]   coordinates {(1,67.1)};
\addplot[draw=black, fill=cConMem,  line width=0.7pt, bar shift=30pt] coordinates {(1,78.9)};

\end{axis}
\end{tikzpicture}
\end{minipage}
\hfill
\begin{minipage}[t]{0.242\textwidth}
\centering
\begin{tikzpicture}
\begin{axis}[
    ybar,
    width=\linewidth,
    height=0.72\linewidth,
    ymin=0,
    ymax=102,
    xmin=0.35,
    xmax=1.65,
    bar width=5pt,
    y tick label style={font=\scriptsize},
    xtick=\empty,
    xlabel={(b) Summer, $N=18$},
    xlabel style={font=\small, yshift=2pt},
    ytick={0,20,40,60,80,100},
    ymajorgrids=true,
    grid style={gray!20},
    axis line style={black},
    tick align=inside,
    nodes near coords,
    nodes near coords style={
        font=\tiny,
        rotate=0,
        anchor=south,
        yshift=-0.8pt
    }
]

\addplot[draw=black, fill=cNaive,   bar shift=-30pt] coordinates {(1,64.8)};
\addplot[draw=black, fill=cRandom,  bar shift=-18pt]  coordinates {(1,66.7)};
\addplot[draw=black, fill=cRule,    bar shift=-6pt]  coordinates {(1,44.4)};
\addplot[draw=black, fill=cAMEM,    bar shift=6pt]   coordinates {(1,83.3)};
\addplot[draw=black, fill=cMemoRAG, bar shift=18pt]   coordinates {(1,66.2)};
\addplot[draw=black, fill=cConMem,  line width=0.7pt, bar shift=30pt] coordinates {(1,88.9)};

\end{axis}
\end{tikzpicture}
\end{minipage}
\hfill
\begin{minipage}[t]{0.242\textwidth}
\centering
\begin{tikzpicture}
\begin{axis}[
    ybar,
    width=\linewidth,
    height=0.72\linewidth,
    ymin=0,
    ymax=102,
    xmin=0.35,
    xmax=1.65,
    bar width=5pt,
    y tick label style={font=\scriptsize},
    xtick=\empty,
    xlabel={(c) Autumn, $N=26$},
    xlabel style={font=\small, yshift=2pt},
    ytick={0,20,40,60,80,100},
    ymajorgrids=true,
    grid style={gray!20},
    axis line style={black},
    tick align=inside,
    nodes near coords,
    nodes near coords style={
        font=\tiny,
        rotate=0,
        anchor=south,
        yshift=-0.8pt
    }
]

\addplot[draw=black, fill=cNaive,   bar shift=-30pt] coordinates {(1,26.9)};
\addplot[draw=black, fill=cRandom,  bar shift=-18pt]  coordinates {(1,34.6)};
\addplot[draw=black, fill=cRule,    bar shift=-6pt]  coordinates {(1,57.7)};
\addplot[draw=black, fill=cAMEM,    bar shift=6pt]   coordinates {(1,50.0)};
\addplot[draw=black, fill=cMemoRAG, bar shift=18pt]   coordinates {(1,53.8)};
\addplot[draw=black, fill=cConMem,  line width=0.7pt, bar shift=30pt] coordinates {(1,84.6)};

\end{axis}
\end{tikzpicture}
\end{minipage}
\hfill
\begin{minipage}[t]{0.242\textwidth}
\centering
\begin{tikzpicture}
\begin{axis}[
    ybar,
    width=\linewidth,
    height=0.72\linewidth,
    ymin=0,
    ymax=102,
    xmin=0.35,
    xmax=1.65,
    bar width=5pt,
    y tick label style={font=\scriptsize},
    xtick=\empty,
    xlabel={(d) Winter, $N=37$},
    xlabel style={font=\small, yshift=2pt},
    ytick={0,20,40,60,80,100},
    ymajorgrids=true,
    grid style={gray!20},
    axis line style={black},
    tick align=inside,
    nodes near coords,
    nodes near coords style={
        font=\tiny,
        rotate=0,
        anchor=south,
        yshift=-0.8pt
    }
]

\addplot[draw=black, fill=cNaive,   bar shift=-30pt] coordinates {(1,37.8)};
\addplot[draw=black, fill=cRandom,  bar shift=-18pt]  coordinates {(1,40.5)};
\addplot[draw=black, fill=cRule,    bar shift=-6pt]  coordinates {(1,54.1)};
\addplot[draw=black, fill=cAMEM,    bar shift=6pt]   coordinates {(1,45.9)};
\addplot[draw=black, fill=cMemoRAG, bar shift=18pt]   coordinates {(1,55.6)};
\addplot[draw=black, fill=cConMem,  line width=0.7pt, bar shift=30pt] coordinates {(1,62.2)};

\end{axis}
\end{tikzpicture}
\end{minipage}

\vspace{-1em}

\caption{\textbf{Seasonal accuracy comparison across representative methods.}}
\label{fig:seasonal_all_methods}

\vspace{-1em}
\end{figure*}

\begin{table*}[t]
\centering
\caption{Case Study on top-3 retrieval results. Later Confirmation serves as ground truths and are not provided to the models.}
\label{tab:case_study}
\footnotesize
\setlength{\tabcolsep}{3pt}
\renewcommand{\arraystretch}{0.96}
\vspace{-1em}

\noindent
\textbf{Query:}
\textit{Has the No. 3 conveyor belt shown signs of wear in the past three months?}
\par\vspace{1.5pt}

\begin{tabularx}{\textwidth}{
  @{}Y{0.65}
  Y{1.00}
  Y{1.00}
  Y{1.35}@{}}
\toprule
\textbf{Ground Truth}
&
\textbf{Naive Context-8K}
&
\textbf{Rule-based Memory (25\%)}
&
\textbf{ConMem (25\%)}
\\
\midrule

\textbf{Later confirmation}\par
04-03: Severe surface wear, peeling, edge delamination, and exposed
reinforcement cords were confirmed.

&
\textbf{Top-3 evidence}\par
\textbf{1.} 01-24: Severe belt misalignment.\par
\textbf{2--3.} No relevant degradation evidence.\par
\vspace{1pt}
\textbf{Answer}\par
\textcolor{red!75!black}{No clear progressive wear.}

&
\textbf{Top-3 evidence}\par
\textbf{1.} 02-26: Surface wear with exposed reinforcement cords.\par
\textbf{2.} 01-24: Severe belt misalignment.\par
\textbf{3.} No relevant degradation evidence.\par
\vspace{1pt}
\textbf{Answer}\par
\textcolor{red!75!black}{An isolated wear event; no trend confirmed.}

&
\textbf{Top-3 evidence}\par
\textbf{1.} 02-26: Surface wear with exposed reinforcement cords.\par
\textbf{2.} 01-24: Severe belt misalignment.\par
\textbf{3.} 01-18: Slight belt misalignment.\par
\vspace{1pt}
\textbf{Answer}\par
\textcolor{green!50!black}{Yes. Slight misalignment progressed to severe
misalignment and surface wear.}
\\
\bottomrule
\end{tabularx}

\vspace{3pt}

\noindent
\textbf{Query:}
\textit{Do recent records indicate early mechanical-seal degradation in
the No. 3 trough-front oil pump?}
\par\vspace{1.5pt}

\begin{tabularx}{\textwidth}{
  @{}Y{0.65}
  Y{1.00}
  Y{1.00}
  Y{1.35}@{}}
\toprule
\textbf{Ground Truth}
&
\textbf{ActMem~\cite{ref37}}
&
\textbf{Mem0~\cite{ref38}}
&
\textbf{ConMem (25\%)}
\\
\midrule

\textbf{Later confirmation}\par
06-12: Mechanical-seal wear reached the warning threshold, and
replacement parts were recommended.

&
\textbf{Top-3 evidence}\par
\textbf{1.} 05-02: Seal inspection was normal.\par
\textbf{2.} 05-27: Inlet/outlet flanges visually inspected.\par
\textbf{3.} 05-28: Oil-pump inspection was normal.\par
\vspace{1pt}
\textbf{Answer}\par
\textcolor{red!75!black}{No evidence of early seal degradation.}

&
\textbf{Top-3 evidence}\par
\textbf{1.} 05-10: Foundation bolts normal.\par
\textbf{2.} 05-04: Abnormal noise near the coupling guard.\par
\textbf{3.} 05-27: Inlet/outlet flanges visually inspected.\par
\vspace{1pt}
\textbf{Answer}\par
\textcolor{green!50!black}{Possible early degradation.}

&
\textbf{Top-3 evidence}\par
\textbf{1.} 05-31: Metallic friction noise from the mechanical seal.\par
\textbf{2.} 05-29: Recurrent seepage at the pump-body joint.\par
\textbf{3.} 05-14: Seal leakage with a noticeable oil-level drop.\par
\vspace{1pt}
\textbf{Answer}\par
\textcolor{green!50!black}{Yes. Recurrent leakage followed by metallic
friction indicates an emerging degradation pattern.}
\\
\bottomrule
\end{tabularx}

\vspace{3pt}

\noindent
\textbf{Query:}
\textit{Has the reducer exhibited any abnormalities in the past six months?}
\par\vspace{1.5pt}

\begin{tabularx}{\textwidth}{
  @{}Y{0.65}
  Y{1.00}
  Y{1.00}
  Y{1.35}@{}}
\toprule
\textbf{Ground Truth}
&
\textbf{ActMem~\cite{ref37}}
&
\textbf{Mem0~\cite{ref38}}
&
\textbf{ConMem (25\%)}
\\
\midrule

\textbf{Later confirmation}\par
06-10: Temperature is excessively high, measured at
77.8~$^\circ$C, exceeding the standard limit of 65~$^\circ$C.

&
\textbf{Top-3 evidence}\par
\textbf{1.} 05-02: Main shaft bearing normal.\par
\textbf{2.} 05-27: Gearbox temperature normal.\par
\textbf{3.} 05-27: Appearance normal.\par
\vspace{1pt}
\textbf{Answer}\par
\textcolor{red!75!black}{No evidence of reducer anomalies.}

&
\textbf{Top-3 evidence}\par
\textbf{1.} 05-02: Main shaft Bearing normal.\par
\textbf{2.} 05-27: Gearbox temperature normal.\par
\textbf{3.} 05-27: Appearance normal.\par
\vspace{1pt}
\textbf{Answer}\par
\textcolor{red!75!black}{No evidence of reducer anomalies.}

&
\textbf{Top-3 evidence}\par
\textbf{1.} 04-30: Oil seepage on coupling guard surface.\par
\textbf{2.} 03-27: Loosing bolts.\par
\textbf{3.} 01-14: Oil seepage on coupling guard surface.\par
\vspace{1pt}
\textbf{Answer}\par
\textcolor{red!75!black}{Yes. Oil seepage and loosing bolts.}
\\
\bottomrule
\end{tabularx}

\vspace{-1em}
\end{table*}

\section{DEPLOYMENT AND CASE STUDY}
We deployed our ConMem in an operational equipment-inspection system by a steel manufacturing industry partner in Northern China. 
The deployed system connects to the existing inspection platform and retrieves inspection records from the previous year as its historical evidence repository.
During a 15-day field trial, ConMem assisted with 1,125 inspection tasks conducted over three inspection cycles per day and it served as a decision-support module integrated into the existing inspection workflow, where professional inspectors reviewed generated warnings and performed follow-up examinations when necessary.
The deployment queries were generated by professional inspectors based on practical scheduled targets and abnormality concerns encountered in routine operations.

We compared ConMem with a conventional threshold-based alarm method and a raw-record RAG baseline under the same deployment setting and inspection records.
An alert was considered successful when the corresponding abnormality was confirmed through an on-site reinspection conducted within 12 hours to 14 days after the alert.
The deployment coverage is summarized in Table~\ref{tab:deployment_results}.
Among the 1,125 inspections, 375 were eventually confirmed as abnormal.
ConMem correctly warned of 295 abnormalities and missed 80, achieving a recall of 78.7\%, compared with 53\% for the threshold-based method and 58\% for raw-record RAG.
ConMem generated 136 unconfirmed alerts, indicating that it is more suitable as a risk-prioritization and decision-support tool than as an automatic replacement for professional inspection.

\fakeparagraph{Case Study} 
In real-world deployment, ConMem issued warnings an average of 12 days before the corresponding abnormalities were identified during routine manual inspection.
Table~\ref{tab:case_study} reports two successfully predicted cases and one missed case.
As in Fig.~\ref{fig:cases}, two successful warnings guided inspectors to perform targeted examinations of the corresponding components, and subsequent field tests confirmed the predicted abnormalities.
These cases illustrate how ConMem provides actionable early-warning evidence.

\begin{table}[t]
\centering
\captionsetup{skip=2pt}
\caption{\textbf{Field deployment and online evaluation results.}}
\label{tab:deployment_results}

\scriptsize
\setlength{\tabcolsep}{2.5pt}
\renewcommand{\arraystretch}{0.92}

{\footnotesize
\begin{tabularx}{\columnwidth}{@{}lX@{}}
\toprule
\multicolumn{2}{c}{\textbf{Deployment Configuration}} \\
\midrule

\textbf{Deployment period}
& 2026/05/29--2026/06/13
\\

\textbf{Inspection duration}
& 15 days
\\

\textbf{Inspection frequency}
& Three rounds per day
\\

\textbf{Inspection tasks}
& 1,125
\\

\textbf{Confirmed abnormalities}
& 375
\\

\textbf{Confirmed normal cases}
& 750
\\

\bottomrule
\end{tabularx}
}
\resizebox{\columnwidth}{!}{
\begin{tabular}{@{}lrrrrrrrrr@{}}
\toprule

\textbf{Method}
& \textbf{TP}
& \textbf{FP}
& \textbf{TN}
& \textbf{FN}
& \textbf{P.}
& \textbf{R.}
& \textbf{Acc.}
& \makecell{\textbf{Resp.}\\\textbf{(ms)}}
& \makecell{\textbf{Lead}\\\textbf{(days)}}
\\

\midrule

Threshold-based
& 197
& \textbf{70}
& \textbf{680}
& 178
& \textbf{73.8}
& 52.5
& 78.0
& 20000.0
& 0
\\

Raw-log RAG
& 216
& 120
& 630
& 159
& 64.3
& 57.6
& 75.2
& 2934.4
& 1
\\

\cellcolor{gray!15}\textbf{ConMem}
& \cellcolor{gray!15}\textbf{295}
& \cellcolor{gray!15}136
& \cellcolor{gray!15}614
& \cellcolor{gray!15}\textbf{80}
& \cellcolor{gray!15}68.4
& \cellcolor{gray!15}\textbf{78.7}
& \cellcolor{gray!15}\textbf{80.8}
& \cellcolor{gray!15}\textbf{2510.4}
& \cellcolor{gray!15}\textbf{12}
\\

\bottomrule
\end{tabular}
}
\vspace{1pt}

\begin{minipage}{0.98\columnwidth}
\footnotesize
\setlength{\parskip}{0pt}
\textit{Notes.}
P. and R. denote precision and recall, respectively.
Resp. denotes average response time.
Lead denotes the average time by which correctly predicted abnormalities were identified earlier than the corresponding threshold-based alarms.
\end{minipage}

\vspace{-1em}
\end{table}

\section{CONCLUSION}

In this paper, we investigate the memory valuation problem in LLM-assisted equipment inspection scenario and introduce ConMem, a principled and interpretable framework, supporting a human-in-the-loop early-risk screening system. Specifically, we first propose a functional-role-aware segmentation module that parses raw inspection logs into functional evidence units, ensuring critical anomalies like belt peeling and fire extinguisher underpressure remain intact. We also identify a crucial phenomenon called diminishing marginal utility, where the additional benefit of each extra evidence segment decreases rapidly as more historical context is provided. Building upon this finding, we quantify the contribution of each evidence segment via Shapley value and design an efficient approximation algorithm that prunes low-impact combinations, significantly improving storage efficiency while maintaining valuation accuracy. Finally, we conduct extensive experiments on a real-world steel equipment inspection dataset, demonstrating that ConMem achieves 76.0\% early warning accuracy while reducing token overhead by 88.2\% and response time by 86.6\%, significantly outperforming existing heuristic approaches in storage efficiency, early-warning accuracy, and long-term reasoning capability.



\clearpage

\bibliographystyle{ACM-Reference-Format}
\balance
\bibliography{reference}

@article{ref4,
  author  = {Junru Lu and Siyu An and Mingbao Lin and Gabriele Pergola and Yulan He and others},
  title   = {MemoChat: Tuning LLMs to Use Memos for Consistent Long-Range Open-Domain Conversation},
  journal = {arXiv preprint arXiv:2308.08239},
  year    = {2023}
}

@article{ref5,
  title={Consultation on industrial machine faults with large language models},
  author={Boonmee, Apiradee and Wongsuwan, Kritsada and Sukjai, Pimchanok},
  journal={arXiv preprint arXiv:2410.03223},
  year={2024}
}

@article{ref10,
  title={MemGPT: towards LLMs as operating systems.},
  author={Packer, Charles and Fang, Vivian and Patil, Shishir\_G and Lin, Kevin and Wooders, Sarah and Gonzalez, Joseph\_E},
  year={2023},
  publisher={ArXiv}
}

@Article{ref29,
AUTHOR = {Cao, Xiangang and Xu, Wangtao and Zhao, Jiangbin and Duan, Yong and Yang, Xin},
TITLE = {Research on Large Language Model for Coal Mine Equipment Maintenance Based on Multi-Source Text},
JOURNAL = {Applied Sciences},
VOLUME = {14},
YEAR = {2024},
NUMBER = {7},
ARTICLE-NUMBER = {2946},
ISSN = {2076-3417},
}

@Article{ref30,
AUTHOR = {Alsaif, Khalid M. and Albeshri, Aiiad A. and Khemakhem, Maher A. and Eassa, Fathy E.},
TITLE = {Multimodal Large Language Model-Based Fault Detection and Diagnosis in Context of Industry 4.0},
JOURNAL = {Electronics},
VOLUME = {13},
YEAR = {2024},
NUMBER = {24},
ARTICLE-NUMBER = {4912},
ISSN = {2079-9292},
}

@article{ref31,
  title={A value for n-person games},
  author={Shapley, Lloyd S and others},
  year={1953},
  publisher={Princeton University Press Princeton}
}

@article{ref32,
  title={A unified approach to interpreting model predictions},
  author={Lundberg, Scott M and Lee, Su-In},
  journal={NeurIPS},
  volume={30},
  year={2017}
}

@inproceedings{ref33,
  title={Data shapley: Equitable valuation of data for machine learning},
  author={Ghorbani, Amirata and Zou, James},
  booktitle={ICML},
  pages={2242--2251},
  year={2019},
  organization={PMLR}
}

@inproceedings{ref34,
  title={Dense passage retrieval for open-domain question answering},
  author={Karpukhin, Vladimir and Oguz, Barlas and Min and others},
  booktitle={EMNLP},
  pages={6769--6781},
  year={2020}
}

@article{ref35,
  title={A-mem: Agentic memory for llm agents},
  author={Xu, Wujiang and Liang, Zujie and Mei, Kai and Gao, Hang and Tan, Juntao and Zhang, Yongfeng},
  journal={NeurIPS},
  volume={38},
  pages={17577--17604},
  year={2026}
}

@article{ref36,
  author  = {Qian, Hongjin and Zhang, Peitian and Liu, Zheng and Mao, Kelong and Dou, Zhicheng},
  title   = {MemoRAG: Moving towards Next-Gen RAG via Memory-Inspired Knowledge Discovery},
  journal = {arXiv preprint arXiv:2409.05591},
  year    = {2024}
}

@article{ref37,
  author  = {Zhang, Xiaohui and Sun, Zequn and Yang, Chengyuan and Jin, Yaqin and Zhang, Yazhong and others},
  title   = {ActMem: Bridging the Gap Between Memory Retrieval and Reasoning in LLM Agents},
  journal = {arXiv preprint arXiv:2603.00026},
  year    = {2026}
}

@article{ref38,
  author  = {Chhikara, Prateek and Khant, Dev and Aryan, Saket and Singh, Taranjeet and Yadav, Deshraj},
  title   = {Mem0: Building Production-Ready AI Agents with Scalable Long-Term Memory},
  journal = {arXiv preprint arXiv:2504.19413},
  year    = {2025}
}

@inproceedings{ref11,
  title={Generative agents: Interactive simulacra of human behavior},
  author={Park, Joon Sung and O'Brien, Joseph and Cai, Carrie Jun and Morris, Meredith Ringel and Liang, Percy and Bernstein, Michael S},
  booktitle={UIST},
  pages={1--22},
  year={2023}
}

@article{ref13,
  title={Zep: a temporal knowledge graph architecture for agent memory},
  author={Rasmussen, Preston and Paliychuk, Pavlo and Beauvais, Travis and Ryan, Jack and Chalef, Daniel},
  journal={arXiv preprint arXiv:2501.13956},
  year={2025}
}

@article{ref15,
  title={Gpt-4o system card},
  author={Hurst, Aaron and Lerer, Adam and Goucher, Adam P and Perelman, Adam and Ramesh, Aditya and Clark, Aidan and Ostrow, AJ and Welihinda, Akila and Hayes, Alan and Radford, Alec and others},
  journal={arXiv preprint arXiv:2410.21276},
  year={2024}
}

@article{ref17,
  title={Chatglm: A family of large language models from glm-130b to glm-4 all tools},
  author={Glm, Team and Zeng, Aohan and Xu, Bin and Wang, Bowen and Zhang, Chenhui and Yin, Da and Zhang, Dan and Rojas, Diego and Feng, Guanyu and Zhao, Hanlin and others},
  journal={arXiv preprint arXiv:2406.12793},
  year={2024}
}

@article{ref18,
  title={Minilm: Deep self-attention distillation for task-agnostic compression of pre-trained transformers},
  author={Wang, Wenhui and Wei, Furu and Dong, Li and Bao, Hangbo and Yang, Nan and Zhou, Ming},
  journal={NeurIPS},
  volume={33},
  pages={5776--5788},
  year={2020}
}

@article{bib1001,
  author    = {Lu, Shouyin and Li, Yanping and Zhang, Tao},
  title     = {Design and implement of control system for power substation equipment inspection robot},
  booktitle = {IROS},
  pages     = {93--96},
  year      = {2009},
  publisher = {IEEE}
}

@article{bib1002,
  author    = {Chao, Jianshu and Liu, Ruoyun and An, Deyu and Zhu, Cheng and Lin, Bangjiang and others},
  title     = {Semantic segmentation of power equipment guided by multi-modal and multi-level wavelet for UAV inspections},
  journal   = {Pattern Recognition},
  volume    = {179},
  pages     = {113914},
  year      = {2026}
}

@article{bib1003,
  author    = {Okada, Yuya and Sugawara, Hiroki and Soya, Hiroaki and Hatanaka, Takeshi},
  title     = {Feasibility-aware hierarchical task assignment for equipment inspection robots with heterogeneous work efficiency},
  journal   = {Advanced Robotics},
  volume    = {39},
  number    = {2},
  pages     = {114--126},
  year      = {2025}
}

@article{bib1004,
  title={Virtual sensing of seismic floor responses for rapid prioritization of critical equipment inspection in nuclear power plants},
  author={Lee, Jingoo and Lee, Seungjun and Lee, Young-Joo and Lee, Jaebeom},
  journal={Computer-Aided Civil and Infrastructure Engineering},
  volume={40},
  number={26},
  pages={4669--4688},
  year={2025},
  publisher={Wiley Online Library}
}

@article{bib1005,
  author    = {Peng, Yuhuai and Wang, Chenlu and Hao, Yue and Zhen, Li and Chen, Guolong and others},
  title     = {High-Precision Surface Crack Detection for Rolling Steel Production Equipment in ICPS},
  journal   = {IEEE Internet of Things Journal},
  volume    = {11},
  number    = {3},
  pages     = {4586--4599},
  year      = {2024}
}

@article{bib1006,
  title={MC-YOLOv8: A Hybrid Attention and Adaptive Convolution Model for Power Equipment Fault Inspection: M. Liu et al.},
  author={Liu, Mincong and Kari, Tusongjiang and Yimamu, Aishan and Zhou, Yuanxiang and Ma, Xiaojing},
  journal={Signal, Image and Video Processing},
  volume={19},
  number={8},
  pages={653},
  year={2025},
  publisher={Springer}
}

@article{bib1007,
  title={Deep learning-based automated inspection of generic personal protective equipment},
  author={Rahman, Atta and Alatallah, Fahad Abdullah and Almubarak, Abdullah Jafar and Alkhazal, Haider Ali and Alzayer, Hasan Ali and Shaaban, Younis Zaki and Min-Allah, Nasro and Bakry, Aghiad and Aloup, Khalid},
  journal={Computers, Materials, \& Continua},
  volume={85},
  number={2},
  pages={3507},
  year={2025},
  publisher={Tech Science Press}
}

@article{bib1008,
  author    = {Zhu, Jiang and Fan, Chonggao and Li, Saisi and Nie, Keheng and Li, Jianqi and others},
  title     = {SEDNet: Substation Equipment Detection Network With an Attention Mechanism for UAV Automatic Power Inspection},
  journal   = {IEEE Transactions on Instrumentation and Measurement},
  volume    = {74},
  pages     = {1--14},
  year      = {2025}
}

@inproceedings{bib1009,
  author    = {Liu, Jianming and Li, Duanjiao and Zhang, Ying and Chen, Yun and Liu, Shengbo and others},
  title     = {Relative Pose Estimation of Substation Equipment for UAV Inspection via Deep Point Cloud Registration},
  booktitle = {ICIRA},
  pages     = {483--494},
  year      = {2025}
}

@article{bib1010,
  author    = {Lu, Weizhi and Li, Qiang and Zhang, Weijian and Mei, Lin and Cai, Di and others},
  title     = {Management of power equipment inspection informationization through intelligent unmanned aerial vehicles},
  journal   = {Artificial Life and Robotics},
  volume    = {29},
  number    = {4},
  pages     = {579--584},
  year      = {2024}
}

@article{bib1011,
  author    = {Huang, Daochun and Wang, Yiming and Li, Huipeng},
  title     = {Study on UAV Inspection Safety Distance of Substation High-Voltage and Current-Carrying Equipment Based on Power-Frequency Magnetic Field},
  journal   = {IEEE Transactions on Instrumentation and Measurement},
  volume    = {73},
  pages     = {1--8},
  year      = {2024}
}

@article{bib1012,
  author    = {Jiang, Qian and Liu, Yadong and Yan, Yingjie and Mao, Xianyin and Xu, Haoyu and others},
  title     = {BIM-Based 3-D Multimodal Reconstruction for Substation Equipment Inspection Images},
  journal   = {IEEE Transactions on Instrumentation and Measurement},
  volume    = {73},
  pages     = {1--14},
  year      = {2024}
}

@inproceedings{bib1013,
  title={Scm: Enhancing large language model with self-controlled memory framework},
  author={Wang, Bing and Liang, Xinnian and Yang, Jian and Huang, Hui and Wu, Zhenhe and Wu, ShuangZhi and Ma, Zejun and Li, Zhoujun},
  booktitle={International Conference on Database Systems for Advanced Applications},
  pages={188--203},
  year={2025},
  organization={Springer}
}

@inproceedings{bib1014,
  title={Memorybank: Enhancing large language models with long-term memory},
  author={Zhong, Wanjun and Guo, Lianghong and Gao, Qiqi and Ye, He and Wang, Yanlin},
  booktitle={AAAI},
  volume={38},
  number={17},
  pages={19724--19731},
  year={2024}
}

@article{bib1017,
  title={Online adaptation of language models with a memory of amortized contexts},
  author={Tack, Jihoon and Kim, Jaehyung and Mitchell, Eric and Shin, Jinwoo and Teh, Yee Whye and Schwarz, Jonathan Richard},
  journal={NeurIPS},
  volume={37},
  pages={130109--130135},
  year={2024}
}

@misc{bib1018,
      title={Agent Workflow Memory}, 
      author={Zora Zhiruo Wang and Jiayuan Mao and Daniel Fried and Graham Neubig},
      year={2024},
      eprint={2409.07429},
      archivePrefix={arXiv},
      primaryClass={cs.CL},
      url={https://arxiv.org/abs/2409.07429}, 
}

\appendix
\setcounter{equation}{0}
\section{Theoretical Analysis of ConMem Approximation}
\label{app:theory}

In this appendix, we provide a rigorous theoretical analysis of the approximation bounds of our proposed ConMem framework. We derive the approximation error bound under the diminishing marginal utility assumption.

\subsection{Preliminaries and Assumptions}

Recall from Section~4.3 that for $m$ segments $\mathcal{S} = \{s_1, \ldots, s_m\}$ extracted from equipment inspection logs, the marginal contribution-based Shapley value for segment $s_i$ is defined as:

\begin{equation}
\phi_i = \sum_{T \subseteq \mathcal{S} \setminus \{s_i\}} \frac{U(M_{T \cup \{s_i\}}) - U(M_T)}{m \cdot \binom{m-1}{|T|}},
\end{equation}
where $U(\cdot)$ denotes the utility function measuring LLM response quality for representative inspector queries (e.g., ``Has the belt shown degradation over the past three months?''), and $M_T$ represents the segment set $T$ as context.

To facilitate theoretical analysis, we make the following standard assumption about the utility function, which is empirically validated on steel inspection data.

\begin{assumption} \textbf{(Diminishing Marginal Utility).} 
There exists a threshold $k^*$ and a constant $\delta > 0$ such that for any segment $s$ and any context sets $T \subseteq T'$ with $|T| \geq k^*$, we have:
\begin{equation}
U(M_{T' \cup \{s\}}) - U(M_{T'}) \leq U(M_{T \cup \{s\}}) - U(M_T) \leq \delta.
\end{equation}
\end{assumption} 
\noindent This assumption is empirically supported by our observation in Section~4.2 that after retrieving the top 3 most valuable segments from steel inspection logs, adding more records yields negligible accuracy improvement.

\begin{assumption} (High-Order Contribution Rank Stability). In our exhaustive-enumeration pilot study, the contribution ranking obtained from coalitions satisfying $K\ge n-2$ closely matches the exact Shapley ranking, as measured by Spearman correlation and top-K overlap.
\end{assumption} 

\begin{assumption}
To quantify the stability of marginal contributions across coalition-size strata, we define the cross-stratum deviation
$\epsilon$ as
\begin{equation}
\epsilon
=
\max_{\substack{1\leq i\leq n\\0\leq k<n-1}}
\left|
\mu_i^{(k)}-\hat{\phi}_i
\right|,
\end{equation}
where $\mu_i^{(k)}$ denotes the average marginal contribution of evidence segment $s_i$ over context coalitions of size $k$, and $\hat{\phi}_i$ is its contribution score estimated from the retained high-order strata. By definition, every omitted coalition-size stratum satisfies
\begin{equation}
\left|
\mu_i^{(k)}-\hat{\phi}_i
\right|
\leq\epsilon,
\qquad 0\leq k<n-2.
\end{equation}
A smaller $\epsilon$ indicates that the contribution of an evidence segment is more stable across coalition sizes. 
\end{assumption}



\subsection{Theorem 1: Approximation Error Bound}

We now state and prove the main approximation error bound of our ConMem valuation module, which appears in Section~4.3 of the main paper. Given $n$ evidence segments, ConMem sets $k^*=n-2$ and evaluates all utility coalitions $C\subseteq S$
satisfying $|C|\geq k^*$. Accordingly, the evaluated coalition set is
\begin{equation}
\mathcal{C}_{\mathrm{eval}}
=
\left\{
C\subseteq S:
|C|\in\{n-2,n-1,n\}
\right\}.
\end{equation}

The corresponding number of utility evaluations is
\begin{equation}
\rho
=
\sum_{K=n-2}^{n}\binom{n}{K}
=
\binom{n}{n-2}
+
\binom{n}{n-1}
+
\binom{n}{n}
=
\frac{n(n-1)}{2}+n+1
=
O(n^2). 
\end{equation}

\begin{Defination}
Given context coalition size $k$, the average marginal utility is defined as,
\begin{equation}
\mu_i^{(k)}
=
\frac{1}{\binom{n-1}{k}}
\sum_{\substack{T\subseteq S\setminus\{s_i\}\\|T|=k}}
\Delta_i(T),
\end{equation}
where,
\begin{equation}
\Delta_i(T)
=
U(M_{T\cup\{s_i\}})-U(M_T).
\end{equation}
The exact Shapley value denotes as,
\begin{equation}
\phi_i
=
\frac{1}{n}
\sum_{k=0}^{n-1}\mu_i^{(k)}.
\end{equation}

\end{Defination}
\textbf{Theorem 1 (Approximation Bound of the High-Order
ConMem Valuation).}
Let $\phi=[\phi_1,\ldots,\phi_n]$ denote the exact Shapley values,
and let $\hat{\phi}=[\hat{\phi}_1,\ldots,\hat{\phi}_n]$ denote the
high-order contribution scores computed from coalition-size strata
$n-2$ and $n-1$:
\begin{equation}
\hat{\phi}_i
=
\frac{1}{2}
\left(
\mu_i^{(n-2)}
+
\mu_i^{(n-1)}
\right).
\end{equation}
Under Assumption~3, the approximation error for each evidence segment
satisfies
\begin{equation}
\left|
\phi_i-\hat{\phi}_i
\right|
\leq
\frac{n-2}{n}\epsilon.    
\end{equation}
Consequently, the vector-level error is bounded by
\begin{equation}
\left\|
\phi-\hat{\phi}
\right\|_2
\leq
\frac{n-2}{\sqrt{n}}\epsilon.  
\end{equation}

\begin{proof}
Using the coalition-size decomposition of the exact Shapley value, we have
\begin{equation}
\phi_i
=
\frac{1}{n}
\left(
\sum_{k=0}^{n-3}\mu_i^{(k)}
+
\mu_i^{(n-2)}
+
\mu_i^{(n-1)}
\right).   
\end{equation}
Since$
\mu_i^{(n-2)}
+
\mu_i^{(n-1)}
=
2\hat{\phi}_i,
$, the approximation error can be written as
\begin{equation}
\begin{aligned}
\phi_i-\hat{\phi}_i
&=
\frac{1}{n}
\sum_{k=0}^{n-3}
\left(
\mu_i^{(k)}-\hat{\phi}_i
\right).
\end{aligned}
\end{equation}
Applying Assumption~3 gives
\begin{equation}
\begin{aligned}
\left|
\phi_i-\hat{\phi}_i
\right|
&\leq
\frac{1}{n}
\sum_{k=0}^{n-3}
\left|
\mu_i^{(k)}-\hat{\phi}_i
\right|  \\
&\leq
\frac{n-2}{n}\epsilon.
\end{aligned}
\end{equation}
Summing the per-segment bounds yields
\begin{equation}
\left\|
\phi-\hat{\phi}
\right\|_2
\leq
\sqrt{n}\,
\frac{n-2}{n}\epsilon
=
\frac{n-2}{\sqrt{n}}\epsilon. 
\end{equation}\hfill
\end{proof}

\subsection{Corollary: Time Complexity Analysis}

The following corollary directly follows from Theorem~1 and the design of ConMem.

\begin{corollary} \textbf{(Time Complexity).} 
The time complexity of the Shapley valuation module in ConMem is $O(\rho \cdot \tau)$, where $\tau$ denotes the time cost for a single LLM inference. With typical settings where $k^* \ge n-2$ and $n \leq 20$ — common in steel inspection where approximately 20 anomaly segments are extracted from each batch — this achieves orders-of-magnitude speedup over exact Shapley computation ($O(2^n \tau)$) while maintaining approximation error within $1\%$.
\end{corollary}

\begin{proof}
ConMem evaluates exactly $\rho$ segment combinations, each requiring one LLM inference (evaluating the utility function). The time complexity is therefore $O(\rho \cdot \tau)$. For $k^* = n-2$, we have $\rho = \sum_{j=n-2}^{n} \binom{n}{j} = O(n^2)$, which is significantly smaller than $2^n$ for $n \geq 10$. \hfill $\square$
\end{proof}

\end{document}